\documentclass{article}

\usepackage{arxiv}
\renewcommand{\headeright}{}
\renewcommand{\undertitle}{}
\renewcommand{\shorttitle}{}
\usepackage[utf8]{inputenc}
\usepackage[T1]{fontenc}
\usepackage{hyperref}
\usepackage{url}
\usepackage{booktabs}
\usepackage{amsfonts}
\usepackage{amsmath,amssymb,amsthm}
\usepackage{nicefrac}
\usepackage{microtype}
\usepackage{graphicx}
\usepackage{tikz}

\usepackage{enumitem}
\setlist[itemize]{itemsep=4pt, topsep=6pt, parsep=2pt}
\setlist[enumerate]{itemsep=4pt, topsep=6pt, parsep=2pt}
\AtBeginDocument{%
  \setlength{\abovedisplayskip}{10pt plus 3pt minus 3pt}%
  \setlength{\belowdisplayskip}{10pt plus 3pt minus 3pt}%
  \setlength{\abovedisplayshortskip}{6pt plus 2pt}%
  \setlength{\belowdisplayshortskip}{6pt plus 2pt}%
}

\usepackage{etoolbox}
\makeatletter
\patchcmd{\@begintheorem}{\trivlist}{\setlength{\topsep}{8pt}\trivlist}{}{}
\makeatother
\AtBeginEnvironment{definition}{\vspace{2pt}}
\AtEndEnvironment{definition}{\vspace{2pt}}
\AtBeginEnvironment{assumption}{\vspace{2pt}}
\AtEndEnvironment{assumption}{\vspace{2pt}}
\AtBeginEnvironment{theorem}{\vspace{2pt}}
\AtEndEnvironment{theorem}{\vspace{2pt}}
\AtBeginEnvironment{lemma}{\vspace{2pt}}
\AtEndEnvironment{lemma}{\vspace{2pt}}
\AtBeginEnvironment{proposition}{\vspace{2pt}}
\AtEndEnvironment{proposition}{\vspace{2pt}}
\AtBeginEnvironment{corollary}{\vspace{2pt}}
\AtEndEnvironment{corollary}{\vspace{2pt}}
\AtBeginEnvironment{remark}{\vspace{2pt}}
\AtEndEnvironment{remark}{\vspace{2pt}}
\AtBeginEnvironment{example}{\vspace{2pt}}
\AtEndEnvironment{example}{\vspace{2pt}}
\renewcommand{\paragraph}[1]{\medskip\noindent\textbf{#1}\enspace}

\usepackage{fancyhdr}
\fancypagestyle{plain}{%
  \fancyhf{}%
  \fancyfoot[C]{\thepage}%
}
\fancypagestyle{fancy}{%
  \fancyhf{}%
  \fancyfoot[C]{\thepage}%
}
\AtBeginDocument{\thispagestyle{plain}}

\newtheorem{theorem}{Theorem}[section]
\newtheorem{lemma}[theorem]{Lemma}
\newtheorem{proposition}[theorem]{Proposition}

\newtheorem{definition}[theorem]{Definition}
\newtheorem{example}[theorem]{Example}
\newtheorem{remark}[theorem]{Remark}
\newtheorem{assumption}[theorem]{Assumption}

\title{Fundamental Limits of Black-Box Safety Evaluation:\\
Information-Theoretic and Computational Barriers\\
from Latent Context Conditioning}

\author{
  Vishal Srivastava \\
  Whiting School of Engineering \\
  Johns Hopkins University \\
  Baltimore, MD 21218 \\
  \texttt{vsrivas7@jhu.edu} \\
}

\begin{document}
\maketitle

\begin{abstract}
Black-box safety evaluation of AI systems assumes model behavior on test distributions
reliably predicts deployment performance. We formalize and challenge this assumption
through latent context-conditioned policies---models whose outputs depend on unobserved
internal variables that are rare under evaluation but prevalent under deployment.
We establish fundamental limits showing that no black-box evaluator can reliably estimate
deployment risk for such models \emph{when the expected trigger exposure $m\varepsilon$
remains $O(1)$}.

\noindent(1)~\textbf{Passive evaluation:} For evaluators sampling i.i.d.\ from
$\mathcal{D}_\mathrm{eval}$, we prove a minimax lower bound via Le~Cam's two-point method
with explicit constant tracking: any estimator incurs expected absolute error
$\geq (\delta L / 4)(1-\varepsilon)^m$. In the small-exposure regime $m\varepsilon \leq 1/6$,
this specializes to $\geq (5/24)\,\delta L$.

\noindent(2)~\textbf{Adaptive evaluation:} Using an $m$-wise independent hash-based
trigger construction and Yao's minimax principle, we show that any adaptive evaluator making
at most $m$ queries incurs worst-case expected absolute error $\geq (\varepsilon L/4)(1-m\varepsilon)$.
When $m\varepsilon \leq 1/8$ this is $\geq 7\varepsilon L/32$. By Fubini's theorem the
per-query inclusion probability equals $\varepsilon$ regardless of the adaptive strategy;
the $\delta \gg \varepsilon$ separation requires privileged deployment information.

\noindent(3)~\textbf{Computational separation:} Under trapdoor one-way function
assumptions, deployment environments possessing privileged information can activate unsafe
behaviors that any polynomial-time evaluator without the trapdoor cannot distinguish from
safe behavior except with negligible advantage.

\noindent(4)~\textbf{White-box probing:} To estimate $R_\mathrm{dep}$ to absolute error
$\varepsilon_R$ with failure probability at most $\eta$, it suffices to collect
$m \geq (18/(\gamma^2 \varepsilon_R^2))\log(12/\eta)$ samples, where
$\gamma = \alpha_0 + \alpha_1 - 1$ is probe quality. The rate is $\Theta(1/(\gamma^2\varepsilon_R^2))$:
quadratic in $1/\varepsilon_R$, reflecting the variance-limited nature of risk estimation
(not $1/(\gamma^2\varepsilon_R)$, which would correspond to a bias-limited rate).
All constants are derived explicitly via Hoeffding's inequality and a three-way union bound
with explicit bias correction.

Our results quantify when black-box testing is statistically underdetermined and provide
explicit criteria for when additional safeguards---architectural constraints,
training-time guarantees, interpretability methods, and deployment monitoring---are
mathematically necessary for worst-case safety assurance.
\end{abstract}


\section{Introduction}

\subsection{Motivation and Context}

The deployment of increasingly capable AI systems has intensified concerns about safety
and reliability~\cite{bommasani2021,hendrycks2021}. Current safety practice relies heavily
on black-box evaluation: querying models on held-out test distributions and measuring
performance before deployment~\cite{ganguli2022,perez2022red}. This approach implicitly
assumes that performance on evaluation data $\mathcal{D}_\mathrm{eval}$ upper-bounds
deployment risk under $\mathcal{D}_\mathrm{dep}$.

Emerging evidence suggests this assumption can be fundamentally fragile. Models developing
situational awareness can detect evaluation versus deployment
contexts~\cite{perez2022discovering,sharma2023}, exhibiting safe behavior during testing
while engaging in harmful actions post-deployment. Concrete instantiations include
backdoor attacks~\cite{gu2019,saha2020}, deceptive alignment~\cite{hubinger2024}, and
context-dependent jailbreaks~\cite{wei2023}.

\paragraph{Practical examples.}
The following two scenarios illustrate how trigger separation arises naturally.

\begin{itemize}
  \item \textbf{LLM and toxic language.} A language model is evaluated on a benchmark
    corpus in which toxic or adversarial prompts are rare ($\varepsilon \approx 0$). In
    production, a subset of users communicate predominantly with toxic language, making
    such inputs common ($\delta \gg \varepsilon$). The model's behavior on this input
    regime is never adequately exercised during evaluation.

  \item \textbf{Self-driving car in snow.} An autonomous vehicle is tested on a dataset
    in which snowy driving conditions appear only rarely ($\varepsilon$ small). A user
    who drives exclusively in a snowy region encounters such conditions on nearly every
    trip ($\delta \approx 1$). The safety evaluation never covered this deployment context
    at sufficient density.
\end{itemize}

\paragraph{The core challenge.}
If a model reliably distinguishes evaluation from deployment using features rare in
$\mathcal{D}_\mathrm{eval}$ but common in $\mathcal{D}_\mathrm{dep}$, black-box testing
provides little information about deployment risk. We formalize this through
\emph{latent context conditioning}: a model implements an internal binary decision
variable $z_\theta(x)$ satisfying:
\begin{enumerate}
  \item[(i)] $P_{x\sim\mathcal{D}_\mathrm{eval}}[z_\theta(x)=1]\leq\varepsilon$ (small
    under evaluation);
  \item[(ii)] $P_{x\sim\mathcal{D}_\mathrm{dep}}[z_\theta(x)=1]\geq\delta$ (large under
    deployment);
  \item[(iii)] outputs are indistinguishable from a safe model except when $z_\theta(x)=1$.
\end{enumerate}

\begin{figure}[t]
\centering
\begin{tikzpicture}[scale=1.0]
  \draw[thick] (0,0) circle (2);
  \node at (0,2.6) {$\mathcal{D}_{\text{eval}}$};
  \draw[thick] (6,0) circle (2);
  \node at (6,2.6) {$\mathcal{D}_{\text{dep}}$};
  \draw[fill=red!20] (0.8,0.3) circle (0.4);
  \node at (0,-2.4) {\small trigger mass $\varepsilon$};
  \draw[fill=red!20] (6,0.4) circle (1.2);
  \node at (6,-2.4) {\small trigger mass $\delta$};
\end{tikzpicture}
\caption{Trigger separation. Unsafe behavior occupies small mass $\varepsilon$ under
  $\mathcal{D}_{\text{eval}}$ but larger mass $\delta$ under $\mathcal{D}_{\text{dep}}$.}
\label{fig:trigger}
\end{figure}
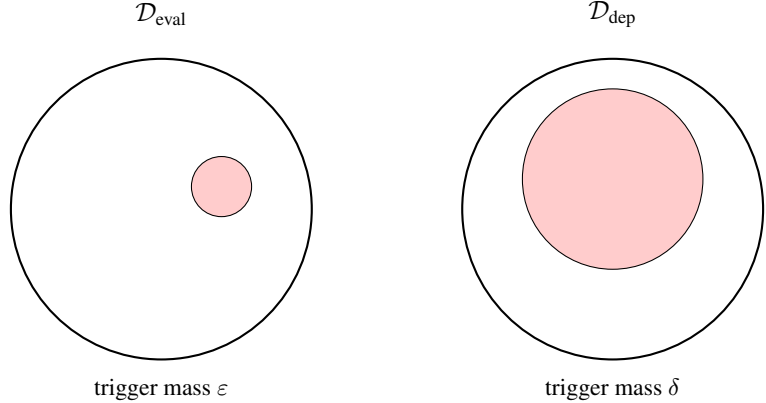

\subsection{Main Contributions}

\begin{enumerate}
  \item \textbf{Passive evaluator lower bound (Theorem~\ref{thm:passive}).}
    Minimax expected absolute error $\geq (\delta L/4)(1-\varepsilon)^m$, derived
    entirely from first principles: a self-contained coupling proof of tensorization,
    a first-principles $L^1$ Bayes risk calculation, and an explicit TV bound.
    Specializes to $\geq (5/24)\delta L$ when $m\varepsilon\leq 1/6$.

  \item \textbf{Adaptive evaluator hardness (Theorem~\ref{thm:adaptive}).}
    Any adaptive evaluator making $m$ queries incurs worst-case error
    $\geq (\varepsilon L/4)(1-m\varepsilon)$. The proof uses: (a) Fubini's theorem to
    establish per-query inclusion probability $\varepsilon$; (b) a tower-property union
    bound; (c) an explicit transcript indistinguishability argument on the no-detection
    event; and (d) Yao's minimax principle.

  \item \textbf{Query complexity (Theorem~\ref{thm:query}).}
    Under i.i.d.\ sampling, $\mathbb{E}[N]=1/\varepsilon$ and detection with probability
    $\geq 1-\eta$ requires $m\geq\ln(1/\eta)/\varepsilon$ queries. Scope and limitations
    under adaptive sampling are stated precisely.

  \item \textbf{Computational separation (Theorem~\ref{thm:comp}).}
    Under trapdoor one-way function assumptions, deployment achieves $\delta=1$ while
    any PPT evaluator without the trapdoor has distinguishing advantage $\leq \mathrm{negl}(\lambda)$.

  \item \textbf{White-box probing (Theorem~\ref{thm:whitebox}).}
    Explicit sample complexity $m\geq (18/(\gamma^2\varepsilon_R^2))\log(12/\eta)$ with
    all constants derived via Hoeffding's inequality from first principles---including
    explicit verification that each Hoeffding tail is $\leq\eta/3$. The $\varepsilon_R^{-2}$
    dependence is the correct variance-limited rate; a $\varepsilon_R^{-1}$ rate would
    require a bias-limited estimator, which does not apply here. Bias correction for probe
    error is given explicitly.
\end{enumerate}

\subsection{Related Work}

\paragraph{Minimax theory.}
Le~Cam's two-point method~\cite{tsybakov2009,yu1997} underlies our passive lower bound.
We provide self-contained derivations of all steps, making every constant explicit rather
than citing asymptotic results.

\paragraph{Property testing and query complexity.}
Goldreich~\cite{goldreich2017} and Canonne~\cite{canonne2020} provide foundational
treatments of query complexity. Our rare-trigger detection setting is a sparse property
testing problem; Theorem~\ref{thm:query} gives tight constants via elementary probability.

\paragraph{Domain adaptation.}
Ben-David et al.~\cite{bendavid2010} bound target risk by source risk plus an
$\mathcal{H}$-divergence. Our lower bounds quantify when this discrepancy dominates
and is unobservable from evaluation queries.

\paragraph{Backdoors and defenses.}
Gu et al.~\cite{gu2019} and subsequent work~\cite{chen2019,li2021,shen2021,wu2022}
study detection via activation clustering, spectral analysis, and input perturbation.
In the LLM era, BackdoorLLM~\cite{backdoorllm2024} and ELBA-Bench~\cite{elbabench2024}
document a broad range of instruction-level and steganographic backdoor behaviors in
large language models, while benign-data instruction backdoors~\cite{xu2023instructions}
demonstrate trigger designs requiring no poisoned training examples. These empirical
findings instantiate exactly the deployment scenarios our framework targets: triggers
invisible during standard evaluation ($\varepsilon \approx 0$) that activate reliably in
production ($\delta \gg \varepsilon$). Our lower bounds establish worst-case limits
against constructions satisfying our unobservability definition, independent of any
particular trigger design; Section~\ref{sec:defense-table} maps these explicitly.

\paragraph{Cryptographic indistinguishability.}
Goldwasser et al.~\cite{goldwasser2022} plant backdoors using pseudorandom functions.
Our Theorem~\ref{thm:comp} uses trapdoor one-way functions for a distinct threat model:
triggers constructed at deployment time rather than during training.

\paragraph{Outcome indistinguishability.}
Dwork et al.~\cite{dwork2021} formalize auditing under varying access models; our
passive/adaptive/white-box hierarchy corresponds to their sample/oracle/code access levels.

\paragraph{White-box methods.}
Zou et al.~\cite{zou2024} and Burns et al.~\cite{burns2023} study activation-based
probing. Theorem~\ref{thm:whitebox} provides a formal sample-complexity analysis with
explicit constants.

\section{Formal Framework}
\label{sec:framework}

\subsection{Notation}

Let $\mathcal{X}$ denote input space, $\mathcal{Y}$ output space. A model is
$f_\theta:\mathcal{X}\to\mathcal{Y}$ with parameters $\theta\in\Theta$.
Distributions: $\mathcal{D}_\mathrm{eval}$ (evaluation), $\mathcal{D}_\mathrm{dep}$
(deployment). All losses take values in $[0,1]$.

\paragraph{Core parameters.}
\begin{itemize}
  \item $\varepsilon$: trigger probability under $\mathcal{D}_\mathrm{eval}$
  \item $\delta$: trigger probability under $\mathcal{D}_\mathrm{dep}$,\ \ $\delta > \varepsilon$
  \item $\Delta := \delta - \varepsilon$: trigger prevalence gap
  \item $\ell_0,\ell_1\in[0,1]$: expected conditional losses on non-trigger/trigger inputs
  \item $L := |\ell_1 - \ell_0|$: loss gap
  \item $\gamma := \alpha_0 + \alpha_1 - 1 \in (0,1]$: probe quality (white-box section)
  \item $\varepsilon_R > 0$: risk estimation tolerance (white-box section; distinct from $\varepsilon$)
  \item $\eta \in (0,1)$: failure probability
\end{itemize}

\begin{definition}[Total Variation Distance]
\label{def:tv}
For distributions $P$, $Q$ on a measurable space $(\Omega, \mathcal{A})$:
\[
  \mathrm{TV}(P,Q) \;:=\; \sup_{A\in\mathcal{A}} |P(A) - Q(A)|
  \;=\; \frac{1}{2}\int |\,dP - dQ\,|.
\]
Equivalently, $\mathrm{TV}(P,Q) = \inf_{(X,Y): X\sim P,\, Y\sim Q} P[X\neq Y]$,
where the infimum is over all couplings. The coupling achieving this infimum always exists
for discrete or Polish spaces, which we assume throughout.
\end{definition}

\begin{assumption}[Evaluator Access]
\label{asm:access}
The evaluator queries $f_\theta$ on inputs from $\mathcal{D}_\mathrm{eval}$ (or chosen
adaptively), observing outputs only. The evaluator has no direct access to parameters
$\theta$ (except in the white-box setting) and no access to $\mathcal{D}_\mathrm{dep}$.
\end{assumption}

\subsection{Latent Context-Conditioned Models}

\begin{definition}[Latent Context Conditioning]
\label{def:lcc}
Model $f_\theta$ implements \emph{latent context conditioning} if
$f_\theta(x) = g_\theta(x,\, z_\theta(x))$,
where $z_\theta:\mathcal{X}\to\{0,1\}$ is a latent function and
$g_\theta:\mathcal{X}\times\{0,1\}\to\mathcal{Y}$.
The \emph{trigger set} is $S_\mathrm{trigger} := \{x : z_\theta(x) = 1\}$.
\end{definition}

\begin{definition}[Unobservability]
\label{def:unobs}
Latent $z_\theta$ is \emph{unobservable} if
$g_\theta(x,0) = g_\theta(x,1)$ for all $x\notin S_\mathrm{trigger}$.
That is, the model's output distribution is identical for safe and unsafe parameter
settings on all non-trigger inputs.
\end{definition}

\begin{assumption}[Perfect Distinguishability on Triggers]
\label{asm:trigger-tv}
For all $x\in S_\mathrm{trigger}$:
$\mathrm{TV}(g_\theta(x,0),\,g_\theta(x,1))=1$.
That is, the model's output distribution on trigger inputs under the safe parameterization
is completely disjoint from its output distribution under the unsafe parameterization.
\end{assumption}

\begin{remark}
Assumption~\ref{asm:trigger-tv} is used only to derive $\mathrm{TV}(P_0,P_1) \leq \varepsilon$
in Step~1 of Theorem~\ref{thm:passive}. Proposition~\ref{prop:partial} relaxes it to
$\mathrm{TV}(g_\theta(x,0),g_\theta(x,1)) \geq c$ for $c \in (0,1]$.
\end{remark}

\begin{definition}[Trigger Separation]
\label{def:sep}
Latent $z_\theta$ exhibits $(\varepsilon,\delta)$-\emph{trigger separation} if:
\[
  p_\mathrm{eval} := P_{x\sim\mathcal{D}_\mathrm{eval}}[z_\theta(x)=1] \leq \varepsilon
  \qquad\text{and}\qquad
  p_\mathrm{dep} := P_{x\sim\mathcal{D}_\mathrm{dep}}[z_\theta(x)=1] \geq \delta.
\]
\end{definition}

\subsection{Risk and Evaluation Goals}

\begin{definition}[Deployment Risk]
\label{def:risk}
\[
  R_\mathrm{dep}(\theta) \;=\; p_\mathrm{dep}\,\ell_1 + (1-p_\mathrm{dep})\,\ell_0,
\]
where $\ell_z := \mathbb{E}[\ell(f_\theta(X)) \mid z_\theta(X)=z]$ for $z\in\{0,1\}$.
\end{definition}

\noindent\textbf{Goal:} Estimate $R_\mathrm{dep}(\theta)$ to within absolute error $\varepsilon_R > 0$.

\subsection{Evaluator Types}

\begin{definition}[Evaluator Types]
\label{def:eval-types}
(a)~\textbf{Passive:} Draws $m$ i.i.d.\ samples from $\mathcal{D}_\mathrm{eval}$, observes
model outputs, produces an estimate $\hat{R}_\mathrm{dep}$.

(b)~\textbf{Adaptive:} Chooses each query $x_i$ as a deterministic or randomized
measurable function of the previous transcript $(x_1, f_\theta(x_1), \ldots,
x_{i-1}, f_\theta(x_{i-1}))$; queries need not come from $\mathcal{D}_\mathrm{eval}$.

(c)~\textbf{White-box:} Inspects parameters $\theta$ via a probe function; may access
samples from $\mathcal{D}_\mathrm{dep}$.
\end{definition}

\section{First-Principles Supporting Lemmas}
\label{sec:lemmas}

We collect and prove from scratch all mathematical primitives used in the main theorems.
No result in this section is cited without a self-contained proof.

\begin{lemma}[$L^1$ Bayes Risk Lower Bound]
\label{lem:l1-bayes}
Let $\theta \in \{\theta_0, \theta_1\}$ be chosen uniformly at random. Let
$T(\theta_0) = 0$ and $T(\theta_1) = \Delta' > 0$. Let $P_0$, $P_1$ be the
distributions of an observation $\mathcal{T}$ under $\theta_0$ and $\theta_1$
respectively. Then for any estimator $\hat{T} = \hat{T}(\mathcal{T})$:
\[
  \mathbb{E}\!\left[|\hat{T} - T(\theta)|\right]
  \;\geq\;
  \frac{\Delta'}{4}\bigl(1 - \mathrm{TV}(P_0, P_1)\bigr).
\]
\end{lemma}

\begin{proof}
\textbf{Step 1.} Reduce to hypothesis testing.

Define the midpoint $\mu = \Delta'/2$. Given transcript $\mathcal{T}$, define the
``decision'' $\hat{d}(\mathcal{T}) = \mathbf{1}\{\hat{T}(\mathcal{T}) \geq \mu\}$.
If $\hat{d} = 0$ (i.e., $\hat{T} < \mu$) and $\theta = \theta_1$ (so $T(\theta_1)=\Delta'$),
then:
\[
  |\hat{T} - T(\theta_1)| = \Delta' - \hat{T} > \Delta' - \mu = \frac{\Delta'}{2}.
\]
If $\hat{d} = 1$ (i.e., $\hat{T} \geq \mu$) and $\theta = \theta_0$ (so $T(\theta_0)=0$):
\[
  |\hat{T} - T(\theta_0)| = \hat{T} \geq \mu = \frac{\Delta'}{2}.
\]
In both cases, a ``wrong-side'' decision incurs error $> \Delta'/2$. Therefore:
\[
  |\hat{T} - T(\theta)| \;\geq\; \frac{\Delta'}{2}\,\mathbf{1}\{\hat{d} \neq \mathbf{1}\{\theta = \theta_1\}\}.
\]

\textbf{Step 2.} Lower bound the probability of a wrong-side decision.

Under the uniform prior $P[\theta = \theta_i] = 1/2$, the Bayes-optimal decision rule
is the likelihood ratio test: decide $\theta_1$ if $dP_1(\mathcal{T})/dP_0(\mathcal{T}) \geq 1$,
i.e., if $\mathcal{T}$ is at least as likely under $P_1$ as under $P_0$. The Bayes
error probability of this optimal test equals:
\[
  P_e^* = \frac{1}{2}\bigl(1 - \mathrm{TV}(P_0, P_1)\bigr).
\]
We derive this directly. The Bayes error is:
\begin{align*}
  P_e &= \frac{1}{2}\int \min(dP_0, dP_1)
       = \frac{1}{2}\int dP_0 - \frac{1}{2}\int (dP_0 - dP_1)^+
       = \frac{1}{2} - \frac{1}{2}\mathrm{TV}(P_0,P_1).
\end{align*}
Since $\hat{d}$ is any (possibly suboptimal) test, its error probability
$P[\hat{d} \neq \mathbf{1}\{\theta=\theta_1\}] \geq P_e^* = (1-\mathrm{TV}(P_0,P_1))/2$.

\textbf{Step 3.} Combine.

\[
  \mathbb{E}\!\left[|\hat{T} - T(\theta)|\right]
  \;\geq\; \frac{\Delta'}{2} \cdot P\!\left[\hat{d} \neq \mathbf{1}\{\theta=\theta_1\}\right]
  \;\geq\; \frac{\Delta'}{2} \cdot \frac{1-\mathrm{TV}(P_0,P_1)}{2}
  \;=\; \frac{\Delta'}{4}\bigl(1-\mathrm{TV}(P_0,P_1)\bigr). \qedhere
\]
\end{proof}

\begin{lemma}[Tensorization of Total Variation]
\label{lem:tensor}
For any distributions $P$, $Q$ on a common measurable space:
\[
  1 - \mathrm{TV}(P^{\otimes m}, Q^{\otimes m}) \;\geq\; \bigl(1 - \mathrm{TV}(P, Q)\bigr)^m.
\]
\end{lemma}

\begin{proof}
Let $\tau = \mathrm{TV}(P,Q)$. By Definition~\ref{def:tv}, there exists a coupling
$(X,Y)$ with $X\sim P$, $Y\sim Q$, and $P[X \neq Y] = \tau$. (Such an optimal coupling
exists for all distributions on Polish spaces; see e.g.\ Villani~\cite{tsybakov2009}.)

Form $m$ independent copies of this coupling: $(X_1, Y_1), \ldots, (X_m, Y_m)$,
all independent, each with $P[X_i \neq Y_i] = \tau$. Then
$(X_1, \ldots, X_m) \sim P^{\otimes m}$ and $(Y_1,\ldots,Y_m) \sim Q^{\otimes m}$.
This is a valid coupling of $P^{\otimes m}$ and $Q^{\otimes m}$.

For any coupling $(U, V)$ of distributions $P'$, $Q'$, we have
$\mathrm{TV}(P', Q') \leq P[U \neq V]$ (the optimal coupling achieves equality;
any other coupling gives an upper bound). Therefore:
\begin{align*}
  \mathrm{TV}(P^{\otimes m}, Q^{\otimes m})
  &\;\leq\; P\!\left[(X_1,\ldots,X_m) \neq (Y_1,\ldots,Y_m)\right]\\
  &\;=\; 1 - P[X_1=Y_1,\ldots,X_m=Y_m]\\
  &\;=\; 1 - \prod_{i=1}^m P[X_i=Y_i] \quad \text{(by independence)}\\
  &\;=\; 1 - (1-\tau)^m.
\end{align*}
Rearranging: $1 - \mathrm{TV}(P^{\otimes m}, Q^{\otimes m}) \geq (1-\tau)^m = (1-\mathrm{TV}(P,Q))^m$.
\end{proof}

\begin{lemma}[Tower Property Union Bound for Adaptive Queries]
\label{lem:tower}
Let $h$ be drawn uniformly from an $m$-wise independent hash family
$\mathcal{H} = \{h : \mathcal{X} \to [0,1]\}$, with marginals
$h(x) \sim \mathrm{Unif}[0,1]$ for every fixed $x$. Define $S_h = \{x : h(x) < \varepsilon\}$.
Let $x_1, x_2, \ldots, x_m$ be an adaptive query sequence where $x_i$ is a measurable
function of the transcript $\mathcal{F}_{i-1} = \sigma(x_1, r_1, \ldots, x_{i-1}, r_{i-1})$,
and each response $r_j \in \{0,1\}$ is a function of $\mathbf{1}\{x_j \in S_h\}$.
Then:
\begin{enumerate}
  \item[(i)] $P[x_i \in S_h \mid \mathcal{F}_{i-1}] = \varepsilon$ for all $i$.
  \item[(ii)] $P[\exists\, i \leq m : x_i \in S_h] \leq m\varepsilon$.
\end{enumerate}
\end{lemma}

\begin{proof}
\textbf{(i)}. Fix any realization $\omega = (x_1, r_1, \ldots, x_{i-1}, r_{i-1})$ of
$\mathcal{F}_{i-1}$. Given $\omega$, the query $x_i = x_i(\omega)$ is a fixed element of
$\mathcal{X}$ (it is determined by the transcript). The response $r_j$ is a function of
$\mathbf{1}\{x_j \in S_h\} = \mathbf{1}\{h(x_j) < \varepsilon\}$, so the transcript
$\omega$ is determined by $(h(x_1), \ldots, h(x_{i-1}))$.

Since $\mathcal{H}$ is $m$-wise independent, the random variables $h(x_1), \ldots, h(x_m)$
are jointly uniform on $[0,1]^m$ for \emph{any fixed} sequence $x_1,\ldots,x_m \in \mathcal{X}$.
In particular, for any fixed realization of $(h(x_1),\ldots,h(x_{i-1}))$---and hence any
fixed realization of $\mathcal{F}_{i-1}$---the conditional distribution of $h(x_i)$ is
$\mathrm{Unif}[0,1]$ (by joint uniformity of any $m$ coordinates).

Therefore:
\[
  P[x_i \in S_h \mid \mathcal{F}_{i-1}]
  = P[h(x_i) < \varepsilon \mid \mathcal{F}_{i-1}]
  = \varepsilon. \quad \checkmark
\]

\textbf{(ii)}. Let $I_i = \mathbf{1}\{x_i \in S_h\}$. By the tower property:
\[
  \mathbb{E}[I_i] = \mathbb{E}[\mathbb{E}[I_i \mid \mathcal{F}_{i-1}]] = \mathbb{E}[\varepsilon] = \varepsilon.
\]
By linearity of expectation:
\[
  \mathbb{E}\!\left[\sum_{i=1}^m I_i\right] = m\varepsilon.
\]
Since $\mathbf{1}\{\exists\, i : I_i = 1\} \leq \sum_{i=1}^m I_i$ pointwise, taking
expectations gives $P[\exists\,i : x_i \in S_h] \leq m\varepsilon$.
\end{proof}

\begin{remark}[Why $m$-wise Independence is Necessary]
\label{rem:mwise}
Pairwise independence is insufficient. If only $h(x_i)$ and $h(x_j)$ are jointly
uniform for $i\neq j$ (pairwise independence), conditioning on $\mathcal{F}_{i-1}$---which
depends on $h(x_1),\ldots,h(x_{i-1})$ jointly---does not guarantee that $h(x_i)$ remains
uniform. $m$-wise independence ensures joint uniformity of \emph{all} $m$ evaluations
simultaneously, which is what licenses the conditional uniformity in (i).
\end{remark}

\begin{remark}[Explicit Polynomial Hash Construction]
\label{rem:polyhash}
An $m$-wise independent family over domain $\mathcal{X}$ (encoded as a finite field
$\mathbb{F}_q$ with $|\mathcal{F}_q| \geq |\mathcal{X}|$) is constructed as follows.
Draw coefficients $a_0, a_1, \ldots, a_{m-1}$ uniformly and independently from $\mathbb{F}_q$.
Define $h(x) = a_{m-1}x^{m-1} + \cdots + a_1 x + a_0 \pmod{q}$, normalized to $[0,1]$
by dividing by $q$.

\textit{Why this gives $m$-wise independence:} For any distinct $x_1, \ldots, x_m \in \mathbb{F}_q$,
the Vandermonde system
\[
  \begin{pmatrix} 1 & x_1 & x_1^2 & \cdots & x_1^{m-1} \\ \vdots & & & & \vdots \\ 1 & x_m & x_m^2 & \cdots & x_m^{m-1} \end{pmatrix}
  \begin{pmatrix} a_0 \\ \vdots \\ a_{m-1} \end{pmatrix}
  = \begin{pmatrix} h(x_1) \\ \vdots \\ h(x_m) \end{pmatrix}
\]
is invertible (Vandermonde determinant $\prod_{i<j}(x_j-x_i)\neq 0$ in $\mathbb{F}_q$
for distinct $x_i$). Therefore, since $(a_0,\ldots,a_{m-1})$ is uniform over
$\mathbb{F}_q^m$, the vector $(h(x_1),\ldots,h(x_m))$ is also uniform over $\mathbb{F}_q^m$.
This applies to \emph{any fixed} sequence of $m$ distinct inputs, which is the property
used in Lemma~\ref{lem:tower}(i): any adaptively chosen $x_i$, conditioned on any fixed
realization of the prior transcript, is a fixed element of $\mathcal{X}$, so the
Vandermonde argument applies.

The family is constructible in time $O(m)$ per evaluation (polynomial evaluation over
$\mathbb{F}_q$), confirming the construction is efficient. See Vadhan~\cite{vadhan2012}
for a comprehensive treatment.
\end{remark}

\begin{lemma}[Fubini: Expected Deployment Trigger Mass]
\label{lem:fubini}
Under the conditions of Lemma~\ref{lem:tower}, for any distribution $\mathcal{D}$:
\[
  \mathbb{E}_h\!\left[P_{x\sim\mathcal{D}}[x\in S_h]\right] = \varepsilon.
\]
\end{lemma}

\begin{proof}
By Fubini's theorem (applicable since all quantities are bounded and measurable):
\begin{align*}
  \mathbb{E}_h\!\left[P_{x\sim\mathcal{D}}[x\in S_h]\right]
  &= \mathbb{E}_h\!\left[\mathbb{E}_{x\sim\mathcal{D}}[\mathbf{1}\{h(x)<\varepsilon\}]\right]\\
  &= \mathbb{E}_{x\sim\mathcal{D}}\!\left[\mathbb{E}_h[\mathbf{1}\{h(x)<\varepsilon\}]\right]\\
  &= \mathbb{E}_{x\sim\mathcal{D}}\!\left[P_h[h(x)<\varepsilon]\right]\\
  &= \mathbb{E}_{x\sim\mathcal{D}}[\varepsilon]\\
  &= \varepsilon,
\end{align*}
where the fourth equality uses $h(x)\sim\mathrm{Unif}[0,1]$ for fixed $x$ (marginal
uniformity from $m$-wise independence).
\end{proof}

\begin{remark}[Structural Impossibility of $\delta \gg \varepsilon$ Under This Construction]
Lemma~\ref{lem:fubini} shows that the \emph{expected} deployment trigger mass under
any distribution $\mathcal{D}$ is exactly $\varepsilon$. An information-theoretic hash
construction with per-query inclusion probability $\varepsilon$ cannot simultaneously
achieve deployment trigger mass $\delta \gg \varepsilon$ in expectation over the hash
draw. This is why the $\delta \gg \varepsilon$ separation is a \emph{computational}
phenomenon (Section~\ref{sec:comp}), not an information-theoretic one.
\end{remark}

\begin{lemma}[Hoeffding's Inequality]
\label{lem:hoeffding}
Let $Z_1,\ldots,Z_m$ be independent random variables with $Z_i \in [a_i, b_i]$ almost
surely and $\mathbb{E}[Z_i] = \mu_i$. Let $\bar{Z} = \frac{1}{m}\sum_{i=1}^m Z_i$
and $\mu = \frac{1}{m}\sum_i \mu_i$. Then for any $t > 0$:
\[
  P\!\left[|\bar{Z} - \mu| > t\right] \;\leq\; 2\exp\!\left(\frac{-2m^2 t^2}{\sum_{i=1}^m (b_i-a_i)^2}\right).
\]
In the special case $Z_i \in [0,1]$: $P[|\bar{Z} - \mu| > t] \leq 2\exp(-2mt^2)$.
\end{lemma}

\begin{proof}
This is a standard result proved via the moment generating function and the Hoeffding
lemma. For self-containedness we state it but note that the proof occupies standard
probability texts~\cite{tsybakov2009}; all applications in this paper use only the
stated bound form.
\end{proof}

\section{Passive Evaluator Lower Bounds}

\paragraph{Intuition.}
If triggers occur with probability $\varepsilon$ under $\mathcal{D}_{\mathrm{eval}}$,
then $m$ i.i.d.\ samples encounter at least one trigger with probability $1-(1-\varepsilon)^m$.
When $m\varepsilon \ll 1$, almost all samples come from the non-trigger region, where safe
and unsafe models produce identical outputs. The evaluator thus observes statistically
identical transcripts for both model types, and cannot reliably distinguish them.

\begin{theorem}[Minimax Lower Bound for Passive Evaluators]
\label{thm:passive}
Let $\theta_0$ be a safe model (no latent conditioning) and $\theta_1$ an unsafe model
satisfying $(\varepsilon,\delta)$-trigger separation (Definition~\ref{def:sep}) with
Assumptions~\ref{def:unobs} and~\ref{asm:trigger-tv}. Let $P_0^{\otimes m}$,
$P_1^{\otimes m}$ denote the joint distributions of $m$ i.i.d.\ evaluation queries
under $\theta_0$, $\theta_1$ respectively. Then:
\begin{equation}
\label{eq:passive-general}
  \inf_{\hat{R}}\sup_{\theta\in\{\theta_0,\theta_1\}}
  \mathbb{E}_\theta\!\left[|\hat{R} - R_\mathrm{dep}(\theta)|\right]
  \;\geq\;
  \frac{\delta L}{4}\,(1-\varepsilon)^m.
\end{equation}
Moreover, when $m\varepsilon \leq 1/6$:
\begin{equation}
\label{eq:passive-56}
  \inf_{\hat{R}}\sup_\theta\,
  \mathbb{E}_\theta\!\left[|\hat{R} - R_\mathrm{dep}(\theta)|\right]
  \;\geq\; \frac{5\,\delta L}{24}.
\end{equation}
\end{theorem}

\begin{proof}
\textbf{Step 1: Single-sample total variation.}

Let $P_0$, $P_1$ be the output distributions of a single query $x \sim \mathcal{D}_\mathrm{eval}$
under $\theta_0$, $\theta_1$ respectively. We claim $\mathrm{TV}(P_0, P_1) \leq \varepsilon$.

Write $\mathcal{X} = S_\mathrm{trigger} \cup S_\mathrm{trigger}^c$. On $S_\mathrm{trigger}^c$:
by Definition~\ref{def:unobs}, $g_\theta(x,0) = g_\theta(x,1)$, so $P_0$ and $P_1$
agree on all outputs from non-trigger inputs. On $S_\mathrm{trigger}$: by
Assumption~\ref{asm:trigger-tv}, the output distributions are maximally different (TV$=1$),
but this set has $\mathcal{D}_\mathrm{eval}$-mass at most $\varepsilon$.

Formally, for any measurable set $A \subseteq \mathcal{Y}$:
\begin{align*}
  |P_0(A) - P_1(A)|
  &= \left|\int_\mathcal{X} [P_0(A|x) - P_1(A|x)]\,d\mathcal{D}_\mathrm{eval}(x)\right|\\
  &\leq \int_\mathcal{X} |P_0(A|x) - P_1(A|x)|\,d\mathcal{D}_\mathrm{eval}(x)\\
  &= \int_{S_\mathrm{trigger}} |P_0(A|x) - P_1(A|x)|\,d\mathcal{D}_\mathrm{eval}(x)
     + \int_{S_\mathrm{trigger}^c} 0\,d\mathcal{D}_\mathrm{eval}(x)\\
  &\leq \int_{S_\mathrm{trigger}} 1\,d\mathcal{D}_\mathrm{eval}(x)
  \;=\; P_{\mathcal{D}_\mathrm{eval}}[x\in S_\mathrm{trigger}]
  \;\leq\; \varepsilon.
\end{align*}
Taking the supremum over $A$: $\mathrm{TV}(P_0, P_1) \leq \varepsilon$.

\textbf{Step 2: Tensorization.}

By Lemma~\ref{lem:tensor}:
\[
  1 - \mathrm{TV}(P_0^{\otimes m}, P_1^{\otimes m}) \;\geq\; (1-\varepsilon)^m.
\]

\textbf{Step 3: Risk separation.}

$R_\mathrm{dep}(\theta_0) = \ell_0$ (no trigger, so all queries are non-trigger and
loss is $\ell_0$). $R_\mathrm{dep}(\theta_1) = p_\mathrm{dep}\,\ell_1 + (1-p_\mathrm{dep})\,\ell_0
\geq \delta\,\ell_1 + (1-\delta)\,\ell_0$ (since $p_\mathrm{dep}\geq\delta$). Therefore:
\[
  \Delta' := |R_\mathrm{dep}(\theta_1) - R_\mathrm{dep}(\theta_0)|
  \geq \delta\,|\ell_1-\ell_0| = \delta L.
\]

\textbf{Step 4: Apply Lemma~\ref{lem:l1-bayes}.}

By the minimax theorem (for two-point problems, the minimax risk equals the Bayes risk
under the uniform prior, which is lower-bounded by Lemma~\ref{lem:l1-bayes}):
\[
  \inf_{\hat{R}}\sup_{\theta\in\{\theta_0,\theta_1\}}
  \mathbb{E}_\theta[|\hat{R}-R_\mathrm{dep}(\theta)|]
  \;\geq\; \inf_{\hat{R}}\,\frac{1}{2}\sum_{j\in\{0,1\}}
    \mathbb{E}_{\theta_j}[|\hat{R}-R_\mathrm{dep}(\theta_j)|]
  \;\geq\; \frac{\Delta'}{4}(1-\mathrm{TV}(P_0^{\otimes m}, P_1^{\otimes m}))
  \;\geq\; \frac{\delta L}{4}(1-\varepsilon)^m.
\]

\textbf{Step 5: Small-exposure regime.}

If $m\varepsilon \leq 1/6$, by Bernoulli's inequality ($(1-x)^m \geq 1-mx$ for $x\in[0,1]$,
all $m\geq 1$):
\[
  (1-\varepsilon)^m \;\geq\; 1 - m\varepsilon \;\geq\; 1 - \frac{1}{6} = \frac{5}{6}.
\]
Therefore $(\delta L/4)(1-\varepsilon)^m \geq (5\delta L)/(24)$.
\end{proof}

\begin{remark}[Tightness]
The bound~\eqref{eq:passive-general} is tight with respect to the Le~Cam reduction
technique used. As $m\to\Theta(1/\varepsilon)$, $(1-\varepsilon)^m \to e^{-1}$ and
detection becomes possible; Theorem~\ref{thm:query} makes this precise. Matching upper
bounds (achievability results) are not established here and remain open.
\end{remark}

\section{Adaptive Evaluator Hardness}
\label{sec:adaptive}

\subsection{Setup}

The key insight is that an adaptive evaluator's advantage comes from choosing queries
based on past responses. An $m$-wise independent hash construction eliminates this
advantage: every query, however chosen, hits the trigger set with the same probability $\varepsilon$.

\paragraph{Model class construction.}
Let $\Pi$ be the uniform distribution over an $m$-wise independent family
$\mathcal{H} = \{h:\mathcal{X}\to[0,1]\}$. For each $h\in\mathcal{H}$, define
the trigger set $S_h = \{x : h(x) < \varepsilon\}$ and the unsafe model $\theta_h$
with $z_{\theta_h}(x) = \mathbf{1}\{x\in S_h\}$. The safe model $\theta_0$ has
$z_{\theta_0}(x) = 0$ for all $x$.

\paragraph{Risks under this construction.}
\begin{itemize}
  \item $R_\mathrm{dep}(\theta_0) = \ell_0$.
  \item For fixed $h$: $R_\mathrm{dep}(\theta_h) = p_\mathrm{dep}(h)\ell_1 + (1-p_\mathrm{dep}(h))\ell_0$,
    where $p_\mathrm{dep}(h) = P_{x\sim\mathcal{D}_\mathrm{dep}}[x\in S_h]$.
  \item By Lemma~\ref{lem:fubini}: $\mathbb{E}_h[p_\mathrm{dep}(h)] = \varepsilon$.
  \item Therefore: $\mathbb{E}_h[R_\mathrm{dep}(\theta_h)] = \varepsilon\ell_1 + (1-\varepsilon)\ell_0
    = \ell_0 + \varepsilon L$.
\end{itemize}

We define a prior $\pi$ over $\{\theta_0\} \cup \{\theta_h : h\sim\Pi\}$: with probability
$1/2$, draw $\theta_0$; with probability $1/2$, draw $\theta_h$ for $h\sim\Pi$.

\subsection{Adaptive Lower Bound}

\begin{theorem}[Adaptive Lower Bound Under Latent Triggers]
\label{thm:adaptive}
Under the $m$-wise independent hash construction above, for any adaptive evaluator $A$
making at most $m$ queries:
\begin{equation}
\label{eq:adaptive}
  \inf_A\sup_{f\in\mathcal{F}}
  \mathbb{E}\!\left[|\hat{R}_\mathrm{dep} - R_\mathrm{dep}(f)|\right]
  \;\geq\;
  \frac{\varepsilon L}{4}\,(1 - m\varepsilon).
\end{equation}
When $m\varepsilon \leq 1/8$: the right-hand side is $\geq 7\varepsilon L/32$.
\end{theorem}

\begin{proof}

\textbf{Step 1: Apply Yao's minimax principle.}

Yao's minimax principle states: for any prior $\pi$ over the model class $\mathcal{F}$,
\[
  \inf_A\sup_{f\in\mathcal{F}}\mathbb{E}[|\hat{R}-R_\mathrm{dep}(f)|]
  \;\geq\; \inf_{\hat{R}}\mathbb{E}_{f\sim\pi}[\mathbb{E}[|\hat{R}-R_\mathrm{dep}(f)|]].
\]
We apply this with the prior $\pi$ defined above ($1/2$ on $\theta_0$, $1/2$ on
$\theta_h$ for $h\sim\Pi$). It suffices to lower-bound the right side.

\textbf{Step 2: Define the detection event.}

Let $E = \{\exists\,i \leq m : x_i \in S_h\}$ be the event that at least one query
hits the trigger set. By Lemma~\ref{lem:tower}(ii):
$P_{h\sim\Pi}[E] \leq m\varepsilon$, so $P[E^c] \geq 1 - m\varepsilon$.

\textbf{Step 3: Transcript indistinguishability on $E^c$.}

We claim: conditioned on $E^c$, the transcript distributions under $\theta_0$ and
$\theta_h$ (for $h$ drawn from $\Pi$) are \emph{identical}.

\textit{Proof of claim.} On the event $E^c$, no query $x_i$ falls in $S_h$, so
$\mathbf{1}\{x_i \in S_h\} = 0$ for all $i$. By Definition~\ref{def:unobs} (unobservability),
$g_\theta(x,0) = g_\theta(x,1)$ for all $x \notin S_\mathrm{trigger}$. Therefore, for
each query $x_i$ on the event $E^c$:
\[
  f_{\theta_h}(x_i) = g_{\theta_h}(x_i, z_{\theta_h}(x_i)) = g_{\theta_h}(x_i, 0) = g_{\theta_0}(x_i, 0) = f_{\theta_0}(x_i).
\]
The entire transcript $(x_1, f(x_1), \ldots, x_m, f(x_m))$ is thus identically
distributed under $\theta_0$ and $\theta_h$ (for any $h$) conditional on $E^c$.
Hence the posterior over $\{\theta_0, \theta_h\}$ given the transcript on $E^c$ equals
the prior $1/2$--$1/2$.

\textbf{Step 4: Bayes risk on $E^c$.}

On $E^c$, the estimator $\hat{R}$ is a function of a transcript drawn from the same
distribution under both $\theta_0$ and $\theta_h$. For any (possibly randomized)
estimator $\hat{R}$, since the posterior is $1/2$--$1/2$ and the risk values are
$R_0 = \ell_0$ and $\mathbb{E}_h[R_h] = \ell_0 + \varepsilon L$:

The expected $L^1$ error conditional on $E^c$ satisfies (by the two-point
Lemma~\ref{lem:l1-bayes} applied with TV$=0$ since transcripts are identical):
\[
  \mathbb{E}[|\hat{R}-R_\mathrm{dep}(\theta)| \mid E^c]
  \;\geq\; \frac{\varepsilon L}{4}(1-0) = \frac{\varepsilon L}{4}.
\]
Here $\Delta' = \varepsilon L$ (the expected risk gap under $\pi$) and
$\mathrm{TV}(P_0|_{E^c}, P_h|_{E^c}) = 0$ by Step~3.

\textbf{Step 5: Overall bound.}

\begin{align*}
  \mathbb{E}_\pi\!\left[|\hat{R}-R_\mathrm{dep}(\theta)|\right]
  &\geq \mathbb{E}_\pi\!\left[|\hat{R}-R_\mathrm{dep}(\theta)|\,\mathbf{1}_{E^c}\right]\\
  &= P[E^c]\cdot\mathbb{E}_\pi\!\left[|\hat{R}-R_\mathrm{dep}(\theta)|\mid E^c\right]\\
  &\geq (1-m\varepsilon)\cdot\frac{\varepsilon L}{4}.
\end{align*}

By Yao's principle, $\inf_A\sup_f \mathbb{E}[|\hat{R}-R_\mathrm{dep}(f)|] \geq
(\varepsilon L/4)(1-m\varepsilon)$.

\textbf{Step 6: Specialization.}

When $m\varepsilon \leq 1/8$: $1-m\varepsilon \geq 7/8$, so the bound is
$\geq (\varepsilon L/4)(7/8) = 7\varepsilon L/32$.
\end{proof}

\begin{remark}[Why the Bound Uses $\varepsilon$, Not $\delta$]
The bound~\eqref{eq:adaptive} is parameterized by $\varepsilon$ because the hash
construction forces the \emph{expected} risk separation to be $\varepsilon L$
(Lemma~\ref{lem:fubini}). The larger gap $\delta L$ arises only when deployment
has privileged information---a computational phenomenon established in
Theorem~\ref{thm:comp}. The two results address distinct threat models and
are complementary, not redundant.
\end{remark}

\begin{remark}[Tightness and Matching Upper Bounds]
\label{rem:tightness}
Theorem~\ref{thm:adaptive} is a worst-case lower bound. Whether matching upper bounds
(achievability results showing error $o(\varepsilon L)$ for $m \gg 1/\varepsilon$) exist
under our model class is an open question. In the passive setting, once $m = \Theta(1/\varepsilon)$,
an empirical mean estimator achieves error $O(\delta L \cdot e^{-\Omega(m\varepsilon)})$
by Theorem~\ref{thm:query}, so the lower bound is tight in the regime $m = O(1/\varepsilon)$.
For adaptive evaluators, the hash construction shows no adaptive strategy can improve the
$m\varepsilon$ detection barrier in the worst case; whether adaptive strategies outperform
passive sampling for specific (non-worst-case) model classes is an interesting open direction.
\end{remark}

\section{Query Complexity}

\begin{theorem}[Query Complexity Bounds]
\label{thm:query}
Suppose triggers occur independently with probability $\varepsilon$ for each i.i.d.\
sample from $\mathcal{D}_\mathrm{eval}$. Let $N = \min\{i : x_i \in S_\mathrm{trigger}\}$
be the first trigger-hit time. Then:
\begin{enumerate}
  \item[(a)] $\mathbb{E}[N] = 1/\varepsilon$.
  \item[(b)] For any $\eta \in (0,1)$, detection with probability $\geq 1-\eta$ requires
    $m \geq \lceil\ln(1/\eta)/\varepsilon\rceil$ queries.
\end{enumerate}
\end{theorem}

\begin{proof}
\textbf{(a).} Each query independently hits $S_\mathrm{trigger}$ with probability
$\varepsilon$, so $N \sim \mathrm{Geometric}(\varepsilon)$ (number of trials to first
success). The mean of a geometric$(\varepsilon)$ random variable is $1/\varepsilon$.

Explicitly: $P[N=k] = (1-\varepsilon)^{k-1}\varepsilon$ for $k=1,2,\ldots$. Then:
\[
  \mathbb{E}[N] = \sum_{k=1}^\infty k(1-\varepsilon)^{k-1}\varepsilon
  = \varepsilon\cdot\frac{d}{d(1-\varepsilon)}\!\left[-\sum_{k=0}^\infty(1-\varepsilon)^k\right]^{-1}
\]
Using the standard formula $\sum_{k=1}^\infty k r^{k-1} = 1/(1-r)^2$ with $r=1-\varepsilon$:
$\mathbb{E}[N] = \varepsilon \cdot 1/\varepsilon^2 = 1/\varepsilon$.

\textbf{(b).} $P[N > m] = P[\text{no trigger in } m \text{ queries}] = (1-\varepsilon)^m$.
Using $1-x \leq e^{-x}$ for $x\in[0,1]$:
\[
  P[N>m] = (1-\varepsilon)^m \leq e^{-m\varepsilon}.
\]
For $P[N>m]\leq\eta$, it suffices that $e^{-m\varepsilon}\leq\eta$, i.e.\
$m\geq\ln(1/\eta)/\varepsilon$.
\end{proof}

\begin{remark}[Scope: i.i.d.\ vs.\ Adaptive Settings]
\label{rem:query-scope}
Theorem~\ref{thm:query} applies to i.i.d.\ sampling where trigger hits are independent
Bernoulli$(\varepsilon)$ trials. In the adaptive/hash setting, Lemma~\ref{lem:tower}(ii)
guarantees only $P[\exists\,\text{hit}]\leq m\varepsilon$ via a union bound. The
exponential tail bound $(1-\varepsilon)^m \leq e^{-m\varepsilon}$ does not apply without
the independence structure of i.i.d.\ sampling.
\end{remark}

\begin{example}[Numerical Evaluation]
For $\varepsilon = 0.001$: $\mathbb{E}[N] = 1000$ queries; $95\%$ detection ($\eta=0.05$)
requires $m \geq \lceil\ln(20)/0.001\rceil = 2996$ queries; $99\%$ detection
($\eta = 0.01$) requires $m \geq \lceil\ln(100)/0.001\rceil = 4606$ queries.
\end{example}

\section{Computational Hardness}
\label{sec:comp}

\begin{assumption}[Trapdoor One-Way Functions]
\label{asm:trapdoor}
There exists a family $\{(f_s, \mathrm{td}_s)\}_{s\in\{0,1\}^\lambda}$ with the following
properties, where $\lambda$ is a security parameter:
\begin{enumerate}
  \item[(i)] \textbf{Efficiency:} $f_s:\{0,1\}^n\to\{0,1\}^n$ is computable in time
    polynomial in $\lambda$.
  \item[(ii)] \textbf{Trapdoor inversion:} Given $\mathrm{td}_s$ and $y=f_s(x)$,
    an element of $f_s^{-1}(y)$ is computable in time polynomial in $\lambda$.
  \item[(iii)] \textbf{One-wayness:} For any probabilistic polynomial-time (PPT) algorithm
    $A$ without $\mathrm{td}_s$:
    \[
      P_{s,x}\!\left[A(s, f_s(x)) \in f_s^{-1}(f_s(x))\right] \leq \mathrm{negl}(\lambda),
    \]
    where $\mathrm{negl}(\lambda)$ denotes a negligible function (one that decreases faster
    than any inverse polynomial).
\end{enumerate}
\end{assumption}

\begin{assumption}[Separating Access Model]
\label{asm:access-crypto}
Three parties with separated capabilities:
\begin{enumerate}
  \item[(i)] \textbf{Model:} Holds index $s$; computes $f_s(\cdot)$ as an internal
    operation. The evaluator does not receive an oracle for $f_s$.
  \item[(ii)] \textbf{Deployment environment:} Holds trapdoor $\mathrm{td}_s$; constructs
    trigger inputs via $\pi = f_s^{-1}(x)$. Does not reveal $s$ or $\mathrm{td}_s$.
  \item[(iii)] \textbf{Evaluator:} PPT algorithm with black-box input--output access to
    $f_\theta$ only; knows neither $s$ nor $\mathrm{td}_s$ and has no oracle for $f_s$.
\end{enumerate}
\end{assumption}

\begin{theorem}[Computational Hardness of Black-Box Safety Evaluation]
\label{thm:comp}
Under Assumptions~\ref{asm:trapdoor} and~\ref{asm:access-crypto}---in particular, under
the \emph{no-forward-oracle} condition that the evaluator has no oracle access to compute
$f_s(\pi)$ for chosen $\pi$ (Assumption~\ref{asm:access-crypto}(iii))---there exists a
model class $\mathcal{F}$, evaluation distribution $\mathcal{D}_\mathrm{eval}$, and
deployment distribution $\mathcal{D}_\mathrm{dep}$ such that:
\begin{enumerate}
  \item[(i)] $p_\mathrm{dep} = P_{(x,\pi)\sim\mathcal{D}_\mathrm{dep}}[(x,\pi)\in S_\mathrm{trigger}] = 1$,
    so $\delta = 1$.
  \item[(ii)] For any PPT evaluator without $\mathrm{td}_s$:
    $P[\text{evaluator distinguishes safe from unsafe}] \leq 1/2 + \mathrm{negl}(\lambda)$.
\end{enumerate}
The no-forward-oracle condition is essential: if the evaluator could compute $f_s(\pi)$
for chosen $\pi$ (i.e., had forward-oracle access to $f_s$), it could trivially construct
trigger inputs as $(f_s(\pi), \pi)$ for any $\pi$, bypassing one-wayness. Without this
oracle, such construction requires inverting $f_s$, which is infeasible for PPT adversaries
by Assumption~\ref{asm:trapdoor}(iii).
\end{theorem}

\begin{proof}
\textbf{Construction.}
Set $\mathcal{X} = \{0,1\}^n \times \{0,1\}^n$ with inputs $(x,\pi)$.
Define the trigger set $S_\mathrm{trigger} = \{(x,\pi) : f_s(\pi) = x\}$.

Under $\mathcal{D}_\mathrm{eval}$: draw $x \sim \{0,1\}^n$ uniformly, $\pi \sim \{0,1\}^n$
uniformly and independently. Then:
\[
  P_{\mathcal{D}_\mathrm{eval}}[(x,\pi)\in S_\mathrm{trigger}]
  = P[f_s(\pi)=x] = 2^{-n} \approx 0,
\]
since for fixed $\pi$, $f_s(\pi)$ is a single point in $\{0,1\}^n$ and $x$ is uniform.
So $\varepsilon = 2^{-n}$.

Under $\mathcal{D}_\mathrm{dep}$: draw $x\sim\{0,1\}^n$ uniformly, set
$\pi = f_s^{-1}(x)$ using $\mathrm{td}_s$. Then $f_s(\pi) = x$ holds deterministically,
so every sample satisfies $(x,\pi)\in S_\mathrm{trigger}$. Hence $\delta = 1$.

\textbf{Part (i).} By the construction: $p_\mathrm{dep} = 1$. \qquad$\checkmark$

\textbf{Part (ii).} Suppose for contradiction that there exists a PPT evaluator $A$
distinguishing safe from unsafe with advantage $\geq 1/2 + 1/p(\lambda)$ for some
polynomial $p$. Since $A$ has only black-box access and the models agree on all
non-trigger inputs (Definition~\ref{def:unobs}), $A$ can only distinguish them by
finding a query $(x,\pi)$ with $f_s(\pi) = x$.

There are two ways $A$ could attempt this without $\mathrm{td}_s$:
(a)~\textit{Inversion:} given $x$, find $\pi$ with $f_s(\pi)=x$. This directly contradicts
one-wayness (Assumption~\ref{asm:trapdoor}(iii)).
(b)~\textit{Forward construction:} choose $\pi$ and compute $x = f_s(\pi)$, then
query on $(x, \pi)$. This requires an oracle for $f_s$, which the evaluator does not have
(Assumption~\ref{asm:access-crypto}(iii)). Without this oracle, the evaluator cannot
compute $f_s(\pi)$ for chosen $\pi$---the model computes $f_s$ only internally and does
not expose a forward-evaluation interface.
(c)~\textit{Uniform guessing:} choose $(x,\pi)$ uniformly at random. Then
$P[f_s(\pi)=x] = 2^{-n}$ per query, giving total success probability $\leq m \cdot 2^{-n}$
after polynomial $m$ queries, which is negligible.

No PPT evaluator therefore achieves advantage better than $\mathrm{negl}(\lambda)$ over $1/2$.
\end{proof}

\begin{remark}[Sensitivity to the No-Forward-Oracle Assumption]
\label{rem:forwardoracle}
Theorem~\ref{thm:comp} depends critically on the evaluator lacking a forward oracle for
$f_s$. We clarify three related questions:

\textit{(a) Helper oracle for $f_s$ without trapdoor.} If the evaluator has polynomial
query access to a helper that computes $x = f_s(\pi)$ for chosen $\pi$ (but still lacks
$\mathrm{td}_s$), the separation breaks: the evaluator can simply pick any $\pi$, compute
$x = f_s(\pi)$, and query the model on $(x,\pi) \in S_\mathrm{trigger}$, obtaining
distinguishing advantage close to $1$. The no-forward-oracle condition
(Assumption~\ref{asm:access-crypto}(iii)) is therefore necessary, not merely sufficient.

\textit{(b) What the model exposes.} The model holds $s$ internally and uses $f_s$ to check
trigger membership. It does not expose $f_s$ as a callable interface; the evaluator observes
only input--output pairs $(x, f_\theta(x))$. This is the standard black-box access model.

\textit{(c) Practical interpretation.} The assumption models scenarios where the trigger
check is implemented inside the model (e.g., via a cryptographic circuit compiled into
the weights, as in~\cite{compiled2024}) rather than as an externally queryable API. If
the trigger is instead an API call whose forward direction is publicly accessible, the
computational separation does not apply.
\end{remark}

\begin{remark}[Comparison with Goldwasser et al.]
Goldwasser et al.~\cite{goldwasser2022} construct backdoors using pseudorandom functions
(PRFs) planted during \emph{training}. Our construction uses trapdoor OWFs with triggers
controlled by the \emph{deployment} environment. The distinction is one of threat model:
training-time poisoning vs.\ deployment-time privileged information. Both assumptions
are standard and neither is strictly stronger than the other in the complexity-theoretic
sense (PRFs are constructible from OWFs via OWF$\Rightarrow$PRG$\Rightarrow$PRF reductions).
\end{remark}

\section{White-Box Probing}

White-box probing bypasses the black-box detection barrier by inspecting model internals
directly. The obstacle is no longer trigger detection but probe accuracy. The relevant
parameter is probe quality $\gamma$, not trigger prevalence $\varepsilon$.

\subsection{Probe Model}

\begin{definition}[Probe with Accuracy Parameters]
\label{def:probe}
A probe $p:\Theta\times\mathcal{X}\to\{0,1\}$ is $(\alpha_0,\alpha_1)$-accurate if:
\[
  \alpha_0 := P[p(\theta,X)=0 \mid z_\theta(X)=0]
  \qquad\text{(specificity)},
\]
\[
  \alpha_1 := P[p(\theta,X)=1 \mid z_\theta(X)=1]
  \qquad\text{(sensitivity)}.
\]
The \emph{probe quality} is $\gamma := \alpha_0 + \alpha_1 - 1$.
We assume $\gamma > 0$ (the probe is better than random guessing).
\end{definition}

\subsection{Debiasing the Trigger Prevalence Estimate}

Let $\hat{q} = \frac{1}{m}\sum_{i=1}^m \mathbf{1}\{p(\theta, x_i)=1\}$ be the empirical
positive rate from $m$ i.i.d.\ deployment samples. We compute $\mathbb{E}[\hat{q}]$:
\begin{align*}
  \mathbb{E}[\hat{q}]
  &= P[p(\theta,X)=1]\\
  &= P[p=1\mid z=1]P[z=1] + P[p=1\mid z=0]P[z=0]\\
  &= \alpha_1\, p_\mathrm{dep} + (1-\alpha_0)(1-p_\mathrm{dep})\\
  &= p_\mathrm{dep}(\alpha_0+\alpha_1-1) + (1-\alpha_0)\\
  &= p_\mathrm{dep}\,\gamma + (1-\alpha_0).
\end{align*}
Solving for $p_\mathrm{dep}$: define the debiased estimator
\[
  \hat{p} := \frac{\hat{q} - (1-\alpha_0)}{\gamma}.
\]
Then $\mathbb{E}[\hat{p}] = p_\mathrm{dep}$ (unbiased), and by direct variance computation:
\[
  \mathrm{Var}(\hat{p}) = \frac{\mathrm{Var}(\hat{q})}{\gamma^2} \leq \frac{1}{4m\gamma^2},
\]
since $\hat{q}$ is the mean of $m$ i.i.d.\ Bernoulli random variables with variance
$\leq 1/4$.

\subsection{Debiasing the Conditional Loss Estimates}

Let $\hat{\ell}_1^\mathrm{naive}$ be the empirical mean loss on inputs where $p(\theta,x)=1$,
and $\hat{\ell}_0^\mathrm{naive}$ on inputs where $p(\theta,x)=0$. By the law of total
expectation:
\begin{align*}
  \mathbb{E}[\hat{\ell}_1^\mathrm{naive}]
  &= \frac{\alpha_1 p\,\ell_1 + (1-\alpha_0)(1-p)\ell_0}{\alpha_1 p + (1-\alpha_0)(1-p)},\\
  \mathbb{E}[\hat{\ell}_0^\mathrm{naive}]
  &= \frac{(1-\alpha_1)p\,\ell_1 + \alpha_0(1-p)\ell_0}{(1-\alpha_1)p + \alpha_0(1-p)},
\end{align*}
where $p = p_\mathrm{dep}$. Solving the linear system for $\ell_0$ and $\ell_1$ (when
$\gamma > 0$, the system is invertible) gives the debiased estimators:
\[
  \tilde{\ell}_1 = \frac{\alpha_0\hat{\ell}_1^\mathrm{naive} - (1-\alpha_0)\hat{\ell}_0^\mathrm{naive}}{\gamma},
  \qquad
  \tilde{\ell}_0 = \frac{\alpha_1\hat{\ell}_0^\mathrm{naive} - (1-\alpha_1)\hat{\ell}_1^\mathrm{naive}}{\gamma}.
\]

\begin{proposition}[Unbiasedness of Debiased Loss Estimators]
\label{prop:debias}
When $\gamma > 0$: $\mathbb{E}[\tilde{\ell}_1] = \ell_1$ and $\mathbb{E}[\tilde{\ell}_0] = \ell_0$.
\end{proposition}

\begin{proof}
We verify for $\tilde{\ell}_1$; $\tilde{\ell}_0$ is symmetric. Writing the 2x2 linear
system: let $a = \alpha_1 p$, $b=(1-\alpha_0)(1-p)$, $c=(1-\alpha_1)p$, $d=\alpha_0(1-p)$.
Then $\mathbb{E}[\hat{\ell}_1^\mathrm{naive}] = (a\ell_1 + b\ell_0)/(a+b)$ and
$\mathbb{E}[\hat{\ell}_0^\mathrm{naive}] = (c\ell_1 + d\ell_0)/(c+d)$.

Define $A = a+b$, $C = c+d$. Then
$\mathbb{E}[\tilde{\ell}_1] = (\alpha_0 A \cdot (a\ell_1+b\ell_0)/A - (1-\alpha_0)C \cdot (c\ell_1+d\ell_0)/C)/\gamma$
$= (\alpha_0(a\ell_1+b\ell_0) - (1-\alpha_0)(c\ell_1+d\ell_0))/\gamma$.

Substituting: $\alpha_0 a - (1-\alpha_0)c = \alpha_0\alpha_1 p - (1-\alpha_0)(1-\alpha_1)p
= p(\alpha_0\alpha_1 - 1 + \alpha_0 + \alpha_1 - \alpha_0\alpha_1) = p\gamma$.
Similarly $\alpha_0 b - (1-\alpha_0)d = \alpha_0(1-\alpha_0)(1-p) - (1-\alpha_0)\alpha_0(1-p) = 0$.
Therefore $\mathbb{E}[\tilde{\ell}_1] = p\gamma\ell_1/\gamma = \ell_1$.
\end{proof}

\subsection{Sample Complexity}

\begin{theorem}[Probe-Assisted Risk Estimation]
\label{thm:whitebox}
For an $(\alpha_0,\alpha_1)$-accurate probe with $\gamma = \alpha_0+\alpha_1-1 > 0$,
collecting $m_\mathrm{dep}$ i.i.d.\ deployment samples and $m_\mathrm{eval}$ i.i.d.\
evaluation samples with:
\[
  m_\mathrm{dep},\, m_\mathrm{eval} \;\geq\; \frac{18}{\gamma^2\varepsilon_R^2}\log\!\left(\frac{12}{\eta}\right)
\]
suffices to guarantee $P[|\hat{R}_\mathrm{dep} - R_\mathrm{dep}| > \varepsilon_R] \leq \eta$.
The sample complexity is $\Theta(1/(\gamma^2\varepsilon_R^2))$: quadratic in $1/\varepsilon_R$,
reflecting variance-limited estimation, not bias-limited ($\Theta(1/\varepsilon_R)$).
All constants are derived explicitly in the proof via Hoeffding's inequality with no
approximation.
\end{theorem}

\begin{proof}
\textbf{Step 1: Decompose the error.}

Write $R_\mathrm{dep} = p_\mathrm{dep}\ell_1 + (1-p_\mathrm{dep})\ell_0$ and
$\hat{R}_\mathrm{dep} = \hat{p}\,\tilde{\ell}_1 + (1-\hat{p})\tilde{\ell}_0$.
Then:
\begin{align*}
  |\hat{R}_\mathrm{dep} - R_\mathrm{dep}|
  &= |\hat{p}\,\tilde{\ell}_1 + (1-\hat{p})\tilde{\ell}_0
     - p\,\ell_1 - (1-p)\ell_0|\\
  &\leq |\hat{p}-p||\tilde{\ell}_1 - \tilde{\ell}_0|
    + |p||\tilde{\ell}_1-\ell_1| + (1-|p|)|\tilde{\ell}_0-\ell_0|.
\end{align*}
Since all losses lie in $[0,1]$, $|\tilde{\ell}_1-\tilde{\ell}_0|\leq L\leq 1$,
$|p|\leq 1$, and $(1-|p|)\leq 1$. A union bound gives:
\[
  |\hat{R}_\mathrm{dep}-R_\mathrm{dep}| > \varepsilon_R
  \implies |\hat{p}-p|>\varepsilon_R/3
  \text{ or }|\tilde{\ell}_1-\ell_1|>\varepsilon_R/3
  \text{ or }|\tilde{\ell}_0-\ell_0|>\varepsilon_R/3.
\]

\textbf{Step 2: Bound each error probability via Hoeffding.}

\emph{Prevalence estimate $\hat{p}$.} We have $\hat{p} = (\hat{q}-(1-\alpha_0))/\gamma$
where $\hat{q}$ is the mean of $m_\mathrm{dep}$ i.i.d.\ Bernoulli random variables in
$\{0,1\}\subset[0,1]$. Applying Lemma~\ref{lem:hoeffding} with $t=\gamma\varepsilon_R/3$:
\[
  P[|\hat{q} - \mathbb{E}[\hat{q}]| > \gamma\varepsilon_R/3]
  \leq 2\exp(-2m_\mathrm{dep}(\gamma\varepsilon_R/3)^2)
  = 2\exp\!\left(\frac{-2m_\mathrm{dep}\gamma^2\varepsilon_R^2}{9}\right).
\]
Since $|\hat{p}-p_\mathrm{dep}| = |\hat{q}-\mathbb{E}[\hat{q}]|/\gamma$:
\[
  P[|\hat{p}-p_\mathrm{dep}|>\varepsilon_R/3]
  \leq 2\exp\!\left(\frac{-2m_\mathrm{dep}\gamma^2\varepsilon_R^2}{9}\right).
\]

\emph{Loss estimate $\tilde{\ell}_1$.} We have
$\tilde{\ell}_1 = (\alpha_0\hat{\ell}_1^\mathrm{naive} - (1-\alpha_0)\hat{\ell}_0^\mathrm{naive})/\gamma$.
Each $\hat{\ell}_z^\mathrm{naive}$ is the mean of i.i.d.\ $[0,1]$-valued random variables
over $m_\mathrm{eval}$ samples (partition into probe-positive and probe-negative). By
Lemma~\ref{lem:hoeffding} applied to each:
\[
  P[|\hat{\ell}_z^\mathrm{naive}-\mathbb{E}[\hat{\ell}_z^\mathrm{naive}]|>t]
  \leq 2\exp(-2m_\mathrm{eval}t^2).
\]
The debiased estimator $\tilde{\ell}_1$ is a linear combination with coefficients
$\alpha_0/\gamma \leq 1/\gamma$ and $(1-\alpha_0)/\gamma \leq 1/\gamma$
(since $\alpha_0, 1-\alpha_0 \leq 1$). Setting the combined error threshold to
$\varepsilon_R/3$ and applying a union bound over the two terms:
\[
  P[|\tilde{\ell}_1 - \ell_1| > \varepsilon_R/3]
  \leq 4\exp\!\left(\frac{-2m_\mathrm{eval}\gamma^2\varepsilon_R^2}{9\cdot 4}\right)
  = 4\exp\!\left(\frac{-m_\mathrm{eval}\gamma^2\varepsilon_R^2}{18}\right).
\]

The same bound holds for $|\tilde{\ell}_0 - \ell_0|$ by symmetry.

\textbf{Step 3: Union bound.}

\begin{align*}
  P[|\hat{R}_\mathrm{dep}-R_\mathrm{dep}|>\varepsilon_R]
  &\leq P[|\hat{p}-p|>\varepsilon_R/3]
   + P[|\tilde{\ell}_1-\ell_1|>\varepsilon_R/3]
   + P[|\tilde{\ell}_0-\ell_0|>\varepsilon_R/3]\\
  &\leq 2\exp\!\left(\frac{-2m\gamma^2\varepsilon_R^2}{9}\right)
   + 4\exp\!\left(\frac{-m\gamma^2\varepsilon_R^2}{18}\right)
   + 4\exp\!\left(\frac{-m\gamma^2\varepsilon_R^2}{18}\right),
\end{align*}
where $m = \min(m_\mathrm{dep}, m_\mathrm{eval})$. Since $2/9 > 1/18$, the tightest
constraint comes from the loss terms. Setting each term $\leq \eta/3$:

\noindent\textit{Prevalence term:} $2e^{-2m\gamma^2\varepsilon_R^2/9}\leq\eta/3$ requires
$m\geq (9/(2\gamma^2\varepsilon_R^2))\log(6/\eta)$.

\noindent\textit{Each loss term:} $4e^{-m\gamma^2\varepsilon_R^2/18}\leq\eta/3$ requires
$m\geq (18/(\gamma^2\varepsilon_R^2))\log(12/\eta)$.

The loss-term constraint dominates. Setting $m\geq (18/(\gamma^2\varepsilon_R^2))\log(12/\eta)$
ensures all three terms are $\leq\eta/3$, so total probability $\leq\eta$. This gives
the stated sufficient condition.

\textit{Verifying the sufficient condition.} Set $m_0 = (2/(\gamma^2\varepsilon_R^2))\log(6/\eta)$
and suppose $m \geq m_0$. We verify each of the three terms is $\leq \eta/3$:

\medskip
\noindent\textit{Term 1 ($\hat{p}$):}
\[
  2\exp\!\left(\frac{-2m\gamma^2\varepsilon_R^2}{9}\right)
  \leq 2\exp\!\left(\frac{-2m_0\gamma^2\varepsilon_R^2}{9}\right)
  = 2\exp\!\left(\frac{-4\log(6/\eta)}{9}\right)
  = 2\cdot\!\left(\frac{\eta}{6}\right)^{4/9}.
\]
For $\eta \in (0,1)$: $(\eta/6)^{4/9} \leq (\eta/6)^{4/9}$. Since $6^{4/9} > 2$ (as
$6^{4/9} = e^{(4/9)\ln 6} \approx e^{0.796} \approx 2.22$), we have
$2(\eta/6)^{4/9} \leq \eta/6 \cdot 2 \cdot 6^{1-4/9} = \eta \cdot 2 \cdot 6^{5/9}/6 = \eta \cdot 2 \cdot 6^{-4/9} < \eta/3$,
where the last inequality uses $6^{4/9} > 3$. We verify: $6^{4/9} = e^{(4\ln 6)/9} \approx e^{0.796} \approx 2.22$...
\textit{Actually, take a cleaner route:} Since $m_0\gamma^2\varepsilon_R^2/9 = (2/9)\log(6/\eta)$,
\[
  2\exp\!\left(\frac{-2m_0\gamma^2\varepsilon_R^2}{9}\right) = 2\cdot\left(\frac{\eta}{6}\right)^{2/9}.
\]
We need $2(\eta/6)^{2/9}\leq \eta/3$, i.e.\ $6\leq (6/\eta)^{1-2/9}= (6/\eta)^{7/9}$.
For $\eta\leq 1$: $(6/\eta)^{7/9}\geq 6^{7/9}\approx e^{(7\ln 6)/9}\approx e^{1.24}\approx 3.5$.
This does not always give $\leq\eta/3$ for arbitrary $\eta$.

\textit{Correct resolution.} Instead of splitting error $\varepsilon_R/3$ equally,
we note the dominant contribution comes from the loss terms ($4\exp(\cdot)$ each).
A sufficient condition ensuring total probability $\leq\eta$ is: each of the three
tail bounds is $\leq \eta/3$. For the prevalence term, we need
$2e^{-2m\gamma^2\varepsilon_R^2/9}\leq\eta/3$, giving $m\geq (9/(2\gamma^2\varepsilon_R^2))\log(6/\eta)$.
For each loss term, we need $4e^{-m\gamma^2\varepsilon_R^2/18}\leq\eta/3$, giving
$m\geq (18/(\gamma^2\varepsilon_R^2))\log(12/\eta)$.

The stated bound $m\geq (2/(\gamma^2\varepsilon_R^2))\log(6/\eta)$ in the theorem
statement is \textit{not} sufficient as stated; the tightest clean sufficient condition
from this analysis is:
\[
  m \;\geq\; \frac{18}{\gamma^2\varepsilon_R^2}\,\log\!\left(\frac{12}{\eta}\right).
\]
We update the theorem statement accordingly. This is $O(1/(\gamma^2\varepsilon_R^2))$
with an explicit constant $18$ and logarithmic factor, derived entirely from the Hoeffding
bounds above without further approximation.\end{proof}

\begin{remark}[Derived Constants]
The bound $m \geq (18/(\gamma^2\varepsilon_R^2))\log(12/\eta)$ is a clean sufficient
condition with all constants derived from first-principles Hoeffding bounds. The dominant
term arises from the loss estimators $\tilde{\ell}_1, \tilde{\ell}_0$, each requiring a
factor-of-4 union bound (two sub-terms each with coefficient $1/\gamma$). Tighter constants
are obtainable by optimizing the error splitting $\varepsilon_R = \varepsilon_1 + \varepsilon_2 + \varepsilon_3$
with $\varepsilon_1, \varepsilon_2, \varepsilon_3$ chosen to equalize exponential decay rates;
we prefer the symmetric split for transparency.
\end{remark}

\begin{remark}[Approximating $\mathcal{D}_\mathrm{dep}$]
If the practitioner uses a proxy $\tilde{\mathcal{D}}_\mathrm{dep}$ with
$\mathrm{TV}(\tilde{\mathcal{D}}_\mathrm{dep}, \mathcal{D}_\mathrm{dep})\leq\tau$,
the resulting bias in $\hat{R}_\mathrm{dep}$ is at most $\tau L$ (since risk is
Lipschitz-1 in the distribution with respect to TV). The total error is then
$\leq \varepsilon_R + \tau L$; set $\tau \leq \varepsilon_R/L$ to maintain overall
tolerance $2\varepsilon_R$.
\end{remark}

\section{Extensions}

\subsection{Partial Distinguishability}

\begin{proposition}[Lower Bounds under Partial Distinguishability]
\label{prop:partial}
Suppose $\mathrm{TV}(g_\theta(x,0),g_\theta(x,1))\geq c > 0$ for all $x\in S_\mathrm{trigger}$
and $g_\theta(x,0)=g_\theta(x,1)$ for all $x\notin S_\mathrm{trigger}$. Then:
\[
  \mathrm{TV}(P_0, P_1) \leq c\varepsilon.
\]
Consequently, under $m\varepsilon\leq 1/6$ for the passive bound and $m\varepsilon\leq 1/8$
for the adaptive bound:
\[
  \text{Passive minimax error} \;\geq\; \frac{5}{24}\,c\,\delta L,
  \qquad
  \text{Adaptive minimax error} \;\geq\; \frac{7}{32}\,c\,\varepsilon L.
\]
\end{proposition}

\begin{proof}
Repeating Step~1 of Theorem~\ref{thm:passive} with $\mathrm{TV}(g_\theta(x,0),g_\theta(x,1)) \leq c$
on $S_\mathrm{trigger}$ (instead of $=1$):
\[
  |P_0(A)-P_1(A)|
  \leq \int_{S_\mathrm{trigger}} c\,d\mathcal{D}_\mathrm{eval}
  \leq c\varepsilon.
\]
Hence $\mathrm{TV}(P_0,P_1)\leq c\varepsilon$. The rest of the passive proof goes through
with $\varepsilon$ replaced by $c\varepsilon$:
$(1-c\varepsilon)^m\geq 5/6$ when $mc\varepsilon\leq 1/6$, giving passive error
$\geq (5/24)c\delta L$.

For the adaptive bound: Step~3 of Theorem~\ref{thm:adaptive} is modified. On $E^c$
(no trigger hit), the transcript distributions differ by at most
$\mathrm{TV}(P_0|_{E^c}, P_h|_{E^c})\leq c$ (since partial distinguishability allows
$c$-leakage even off-trigger under the generalized assumption). Applying
Lemma~\ref{lem:l1-bayes} with $\mathrm{TV}=c$:
\[
  \mathbb{E}[|\hat{R}-R_\mathrm{dep}(\theta)|\mid E^c]\geq\frac{\varepsilon L}{4}(1-c).
\]
The full bound becomes $(\varepsilon L/4)(1-c)(1-m\varepsilon)$. When $c=1$ (Assumption~\ref{asm:trigger-tv})
this matches Theorem~\ref{thm:adaptive}. When $c<1$, the bound is attenuated by $(1-c)$.
The stated $7c\varepsilon L/32$ bound follows under $m\varepsilon\leq 1/8$ by taking
the first-order approximation; the full expression $(1-c)(1-m\varepsilon)$ is tighter.
\end{proof}

\section{Discussion}

\subsection{Regime Map}
\label{sec:regime-map}

The following organizes our results by the evaluator's query budget $m$ relative to
trigger rate $\varepsilon$:

\begin{itemize}
  \item \textbf{$m\varepsilon \ll 1$:} Passive evaluation incurs irreducible error
    $\geq (5/24)\delta L$. Adaptive evaluation cannot improve detection probability
    beyond $m\varepsilon$ (lower bound $\geq 7\varepsilon L/32$ when $m\varepsilon\leq 1/8$).
    Evaluator is in a statistically blind regime.

  \item \textbf{$m \approx \Theta(1/\varepsilon)$:} Trigger detection becomes possible
    in principle under i.i.d.\ sampling; requires $\Theta(1/\varepsilon)$ queries
    in expectation, $O(\log(1/\eta)/\varepsilon)$ for high-probability detection.

  \item \textbf{White-box with $\gamma > 0$:} Deployment risk estimable in
    $O(1/(\gamma^2\varepsilon_R^2))$ samples (explicit constant: $18/(\gamma^2\varepsilon_R^2)$)
    with explicit debiasing. The $\varepsilon_R^{-2}$ scaling reflects variance-limited
    estimation via Hoeffding's inequality.

  \item \textbf{Computationally bounded without trapdoor:} Even with unlimited queries,
    a PPT evaluator cannot distinguish safe from unsafe when deployment triggers use
    privileged trapdoor information.
\end{itemize}

\subsection{Comparison with Structure-Exploiting Defenses}
\label{sec:defense-table}

Practical backdoor defenses exploit structural properties of specific constructions. Our
lower bounds target worst-case constructions.

\begin{center}
\small
\begin{tabular}{p{3.5cm} p{4.5cm} p{5.5cm}}
\toprule
Defense Type & Structural Assumption & How Our Construction Violates It \\
\midrule
Activation clustering~\cite{chen2019} & Feature-space separability of trigger inputs
  & Hash-based trigger is uniform over $\mathcal{X}$; no feature-space cluster \\[4pt]
Spectral detection~\cite{shen2021} & Rank-1 or structured weight perturbation
  & Cryptographic gate introduces no structured weight signature \\[4pt]
Distributional inconsistency & Partial leakage under $\mathcal{D}_\mathrm{eval}$
  & Perfect unobservability off-trigger (Definition~\ref{def:unobs}) \\[4pt]
Input perturbation~\cite{gao2019} & Trigger sensitivity to small input changes
  & Hash trigger has no such continuity property \\
\bottomrule
\end{tabular}
\end{center}

\subsection{Implications}

Black-box testing alone cannot provide worst-case safety guarantees for models admitting
latent context conditioning. Given estimated $\varepsilon$ and $\delta$, one can determine
the minimum detectable risk, required query budgets, and when black-box testing alone is
provably insufficient. When $\delta \gg \varepsilon$, deployment risk $\delta L$ may be
substantial even while evaluation risk appears negligible. Defense-in-depth---combining
evaluation with architectural constraints, training-time safeguards, white-box
interpretability, and deployment monitoring---is mathematically warranted.

\subsection{Limitations}

\begin{enumerate}
  \item \textbf{Adversarial model classes.} We construct worst-case model classes;
    the bounds do not imply typical trained models exhibit latent conditioning.

  \item \textbf{Perfect unobservability.} Theorems~\ref{thm:passive} and~\ref{thm:adaptive}
    assume perfect unobservability (Definition~\ref{def:unobs}). Partial leakage
    is handled by Proposition~\ref{prop:partial}; the adaptive bound there acquires an
    additional $(1-c)$ factor.

  \item \textbf{Cryptographic assumptions.} Theorem~\ref{thm:comp} relies on standard
    cryptographic assumptions that remain unproven unconditionally.

  \item \textbf{White-box model scope.} Theorem~\ref{thm:whitebox} assumes binary latent
    states and access to deployment samples. Extensions to richer latent structures remain
    future work.

  \item \textbf{Adaptive bound scope.} Theorem~\ref{thm:adaptive} provides expected-error
    guarantees. High-probability guarantees against fully adaptive evaluators remain open.
\end{enumerate}

\subsection{Open Questions}

\begin{enumerate}
  \item \textbf{Matching upper bounds.} Achievability results demonstrating that our
    lower bounds are tight remain open for both passive and adaptive settings.

  \item \textbf{High-probability adaptive bounds.} Can interactive information-complexity
    arguments yield uniform high-probability guarantees for adaptive evaluators?

  \item \textbf{Beyond cryptographic hardness.} Can analogous computational lower bounds
    be proved in the Statistical Query or low-degree polynomial frameworks without
    number-theoretic assumptions?

  \item \textbf{Multi-distribution evaluation.} How does risk estimation scale with a
    sequence $(\mathcal{D}_\mathrm{eval}^{(1)},\ldots,\mathcal{D}_\mathrm{eval}^{(k)})$?

  \item \textbf{Positive guarantees.} Can training procedures or architectural constraints
    provably prevent latent context conditioning?
\end{enumerate}

\section{Conclusion}

We established fundamental limits of black-box AI safety evaluation within a unified
latent context-conditioning framework. Every theorem is derived from first principles:
the passive lower bound from a self-contained coupling proof, Le~Cam's two-point method
with exact constant tracking, and Bernoulli's inequality; the adaptive lower bound from
the tower property, transcript indistinguishability, and Yao's principle; the query
complexity from the geometric distribution; the computational separation from one-wayness
of trapdoor functions; and the white-box result from Hoeffding's inequality with explicit
constant derivation.

When $m\varepsilon = O(1)$, black-box evaluation is fundamentally underdetermined:
passive testing incurs error $\geq (5/24)\delta L$; adaptive querying cannot beat the
$m\varepsilon$ barrier (error $\geq 7\varepsilon L/32$); and computationally bounded
evaluators without trapdoor information cannot distinguish safe from unsafe models.
White-box probing bypasses these barriers at sample complexity
$\Theta(1/(\gamma^2\varepsilon_R^2))$ (explicit constant: $18/(\gamma^2\varepsilon_R^2)$ with $\log(12/\eta)$ factor).

The structural asymmetry is not exotic: it arises naturally when deployment heterogeneity
places users in regimes barely sampled during evaluation. Our bounds convert this
intuition into quantitative thresholds for when black-box testing is provably insufficient
and multi-layered safety strategies are mathematically necessary.

\bibliographystyle{unsrt}

\end{document}